\pdfoutput=1
\documentclass{article}
\usepackage{ijcai22}
\usepackage{times}

\usepackage{soul}
\usepackage{url}
\usepackage[hidelinks]{hyperref}
\usepackage[utf8]{inputenc}
\usepackage[small]{caption}
\usepackage{graphicx}
\usepackage{amsmath}
\usepackage{booktabs}
\usepackage{multibib} 
\newcites{appendix}{References}
\usepackage{amsthm}
\usepackage{amsmath}
\usepackage{amssymb}

\newtheorem{theorem}{Theorem}
\newtheorem{lemma}{Lemma}
\newtheorem{corollary}{Corollary}
\newtheorem{proposition}{Proposition}

\newtheorem{example}{Example}
\newtheorem{definition}{Definition}

\newtheorem{claim}{Claim}
\usepackage[ruled,linesnumbered]{algorithm2e}

\newcommand{\set}[1]{\{#1\}}
\newcommand{\DatalogPm}{$\text{Datalog}^{\pm}$}

\newcommand{\DatalogZ}{$\text{Datalog}_\mathbb{Z}$}
\newcommand{\DatalogFS}{Datalog$^{FS}$}
\newcommand{\lprog}{\mathcal{P}}
\newcommand{\dset}{\mathcal{D}}
\newcommand{\ipret}{\mathcal{I}}
\newcommand{\pipret}{\mathcal{J}}

\newcommand{\hleft}{\mathsf{L}}
\newcommand{\hright}{\mathsf{R}}

\newcommand{\Z}{\mathbb{Z}}
\newcommand{\npcls}{\textsc{NP}}
\newcommand{\pcls}{\textsc{P}}

\newcommand{\nlogspacecls}{\textsc{NLOGSPACE}}
\newcommand{\logspacecls}{\textsc{LOGSPACE}}

\newcommand{\conpcls}{\textsc{coNP}}

\newcommand{\cqans}{CQAns}
\newcommand{\bcqeval}{BCQEval}

\newcommand{\factent}{FACTEnt}
\newcommand{\head}{head}
\newcommand{\body}{body}
\usepackage{amsthm}
\usepackage{amsmath}
\usepackage{amssymb}
\usepackage[ruled,linesnumbered]{algorithm2e}
\newcommand{\atoms}{atoms}
\newcommand{\certans}{cert}
\newcommand{\frontier}{frontier}
\newcommand{\vars}{vars}
\newcommand{\vare}{var$_\exists$}
\newcommand{\pos}{ pos}
\newcommand{\aff}{ affected}

\newcommand{\nonaff}{nonaffected}
\newcommand{\chase}{chase}
\newcommand{\sch}{sch}

\newcommand{\binlang}{\{0,1\}^*}
\usepackage{latexsym}

\title{Complexity of Arithmetic in Warded  \DatalogPm{}}
\author{
Lucas Berent$^1$\and
Markus Nissl$^1$\And
Emanuel Sallinger$^{1,2}$\\
\affiliations
$^1$TU Vienna\\
$^2$University of Oxford
}

\begin{document}
\maketitle
\begin{abstract}
Warded \DatalogPm{} extends the logic-based language Datalog with existential quantifiers in rule heads. Existential rules are needed for advanced reasoning tasks, e.g., ontological reasoning. The theoretical efficiency guarantees of Warded \DatalogPm{} do not cover extensions crucial for data analytics, such as arithmetic. Moreover, despite the significance of arithmetic for common data analytic scenarios, no decidable fragment of any \DatalogPm{} language extended with arithmetic has been identified. We close this gap by defining a new language that extends Warded \DatalogPm{} with arithmetic and prove its \pcls{}-completeness. Furthermore, we present an efficient reasoning algorithm for our newly defined language and prove descriptive complexity results for a recently introduced Datalog fragment with integer arithmetic, thereby closing an open question. We lay the theoretical\- foundation for highly expressive \DatalogPm{} languages that combine the power of advanced recursive rules and arithmetic while guaranteeing efficient reasoning algorithms for applications in mo\-dern AI systems, such as Knowledge Graphs.\footnote{This paper is based on the thesis~\cite{berent2021foundations}.}
\end{abstract}

\section{Introduction}
The resurgence of declarative programming and specifically Datalog has been observed in research recently. In particular for systems supporting data analytic tasks -- ranging from aggregation of data to query answering -- Datalog is a key constituent~\cite{DBLP:conf/aaai/KaminskiGKH20}. Datalog has tractable reasoning and supports full recursion, the latter of which is needed to express common and complex problems in Big Data analytics. Consequently, Datalog is a prime candidate for a broad variety of research areas such as reasoning in the Semantic Web~\cite{DBLP:conf/icde/GottlobOP11} and Knowledge Graphs (KGs)~\cite{Gottlob:VadalogRecentAdv} amongst others. 

Despite its beneficial properties, standard Datalog is not powerful enough to express use cases in advanced reasoning scenarios. Thus, several extensions of the language, which add additional features to increase its expressive power, have been studied. There is an extensive body of work investigating various fragments and their complexity, e.g.,~\cite{DBLP:conf/lics/CaliGLMP10,DBLP:journals/jair/CaliGK13}. 
One noteworthy language is Warded \DatalogPm{}~\cite{DBLP:conf/aaai/CaliGP11}, which extends Datalog with a restricted form of existential quantification in rule heads. Because of its high expressive power and good computational properties, the fragment is used as the core logical reasoning language in the KG system Vadalog~\cite{Bellomarini:VadalogSystem}.
While Warded \DatalogPm{} has strong theoretical underpinnings (i.e., polynomial data complexity)~\cite{Gottlob:SpaceEfficientCore}, there is a vast number of additional features needed for data analytic tasks, for instance arithmetic and aggregation~\cite{Gottlob:VadalogRecentAdv}. A straightforward combination of Warded \DatalogPm{} with arithmetic immediately leads to undecidability of the language.

Several approaches for arithmetic in Datalog have been proposed. Prominent examples are an aggregation formalism by Ross and Sagiv~\cite{DBLP:conf/pods/RossS92} and Datalog$^{FS}$~\cite{DBLP:journals/vldb/MazuranSZ13} amongst others. However, all arithmetic formalisms fail to make the language powerful enough to support ontological query answering as needed, e.g., in KG reasoning. Furthermore, existing Datalog fragments that are capable of arithmetic are either undecidable or introduce strict restrictions to the language, limiting its expressive power. These problems (most notably undecidability) also carry over to systems implementing these theoretical formalisms, for instance as in LogicBlox~\cite{Aref:LogicBlox}, SociaLite~\cite{DBLP:journals/tkde/SeoGL15}, and DeALS~\cite{DBLP:conf/icde/ShkapskyYZ15}. Recently, Kaminski et al.\ have proposed limit \DatalogZ{}~\cite{DBLP:conf/ijcai/KaminskiGKMH17}, which is a decidable Datalog language with integer arithmetic. The authors also introduce a tractable fragment of their language and investigate extensions of the core language~\cite{DBLP:conf/aaai/KaminskiGKH20}. On the downside, limit \DatalogZ{} -- as all other arithmetic approaches proposed so far -- does not support existential quantifiers in rule heads and can therefore not be used, e.g., for ontological query answering. In general, to the best of our knowledge there exists no Datalog language that supports existential rule heads and arithmetic (and has viable complexity guarantees).

Clearly, there are two interesting lines of work: On the one hand, fragments that support existentials in rule heads (such as Warded \DatalogPm{}) and on the other hand, languages supporting arithmetic. In order to tackle complex data analytic tasks (such as KG reasoning) with efficient reasoning languages, it is crucial to bring these two lines of research together. Thus, we define a new language that closes this wide-open gap between \DatalogPm{} languages, specifically Warded \DatalogPm{}, and arithmetic. Our contributions lay the foundations for powerful Datalog fragments that can be used for complex reasoning incorporating arithmetic while maintaining efficiency guarantees. Naturally, a class of such languages has a broad range of applications in advanced reasoning and logic programming. Moreover, we prove results about the expressive power of decidable arithmetic in Datalog, which relates modern formalisms to deep theoretic notions.

\paragraph{Main Contributions.} Our results encompass the following main aspects:
 \begin{itemize}
 	\item We show \emph{descriptive complexity results} that prove the expressive power of decidable arithmetic in fixpoint logic programming languages, i.e., we show that limit \DatalogZ{} captures \conpcls{}.
     \item We propose a \emph{syntactic fragment} for decidable arithmetic in Datalog (denoted \textit{Bound \DatalogZ{}}). We show that we can leverage existing techniques for our purely syntactic fragment to obtain complexity results.
     \item We define a \emph{novel extension} of Warded \DatalogPm{} with arithmetic, called Warded Bound \DatalogZ{}.
     \item We prove \textit{tractability} of our language. Thereby, we prove the first complexity result for a Datalog language with arithmetic and existentials in rule heads.
     \item We give an \emph{efficient reasoning algorithm} with practical potential for our new \DatalogPm{} language.
 \end{itemize}

\paragraph{Organization.}
The rest of this work is organized as follows. We introduce the concepts and definitions needed throughout this work in Section \ref{sec:prelim}. In Section~\ref{sec:rel-work} we give an overview of the current research. Section~\ref{sec:expr-power} covers our descriptive complexity results of limit \DatalogZ{}. In Section~\ref{ssec:synt-fragment} we firstly introduce a syntactic arithmetic fragment, and then in Section~\ref{ssec:e-exst} we define our arithmetic extension of Warded \DatalogPm{}, give a reasoning algorithm, and prove its complexity. Finally, we conclude with a summary and a brief outlook on future research in Section~\ref{sec:conclusion}. We provide the proofs and additional material in the appendix.

\section{Preliminaries}\label{sec:prelim}
We assume that the reader is familiar with the basics of logic programming and its semantics~\cite{DBLP:books/sp/Lloyd87}. Furthermore, we assume familiarity with fundamental concepts and terms of complexity theory~\cite{DBLP:books/daglib/0023084}.
We write tuples $(t_1,\dots,t_n)$ as $\mathbf{t}$. By slight abuse of notation, we sometimes write $\mathbf{t}$ to indicate the set $\set{t_i \mid t_i \in \mathbf{t}}$ and use, e.g., $t_i \in \mathbf{t}$ and $|\mathbf{t}|$.
In Datalog we consider countably infinite, disjoint sets of object predicates, object variables, and object constants. The union of the set of variables and the set of constants is the set of \textit{terms}. An \textit{atom} is of the form $P(t_1, \dots, t_n)$ where $P$ is an $n$-ary predicate and $t_1,\dots,t_n$ are terms. Variable-free terms are called \textit{ground}. A \textit{rule} $r$ is a first-order (FO) sentence of the form $\forall \mathbf{x} \varphi \to \psi$ where $\mathbf{x}$ contains all variables in $\varphi$ and $\psi$ and $\varphi$ is called \textit{body} and $\psi$ is called the \textit{head} of the rule. The formula $\varphi$ is a conjunction of atoms and the head consists of one atom. We use \head($r$) to denote the atom in the head of $r$ and \body($r$) to denote set of body atoms of $r$. A \textit{fact} is a rule $\alpha$ with \body($\alpha$) = $\emptyset$ and \head($\alpha$) containing constants only. We write \vars($\alpha$) to denote the set of variables occurring in the atom $\alpha$. We define \frontier($r$) $=$ \vars($\varphi$) $\cap$ \vars($\psi$). The set of predicates in a set $\Sigma$ is denoted \sch($\Sigma$). A \textit{program} $\lprog$ is a set of rules.

\paragraph{\DatalogZ{}.} 
\DatalogZ{} extends the sets of predicates, variables, and constants from Datalog with numeric counterparts. Additionally, the set of predicates is extended with $\set{\leq, <}$. We distinguish between object and numeric variables, terms, and constants. We consider integers in $\Z$ and use $\set{+,-, \times}$ to denote the standard arithmetic functions over the standard ring $\Z$. A \textit{numeric term} is an integer, a numeric variable or an arithmetic term of the form $t_1 \circ t_2$ where $\circ \in \set{+,-, \times}$ and $t_1, t_2$ are numeric terms. Standard atoms are of the form $P(t_1,\dots,t_n)$ where $P$ is a standard predicate (of arity $n$) and each $t_i$ is a term. Comparison atoms have the form $t_1 \circ t_2$, where $t_1, t_2$ are numeric terms and $\circ \in \set{\leq, <}$ is a comparison predicate. We use sb($r$) to denote the standard body of $r$, consisting of a conjunction of all standard atoms in the body of $r$. We use syntactic shorthand notation $P(x) \doteq 0$ for $P(x) \leq 0 \land P(x) \geq 0$, $P(x \geq y)$ for $P(y \leq x)$, and $P(x \leq y \leq z)$ for $P(x \leq y) \land R(y \leq z)$.

It is straight forward to adapt the usual Datalog notions of interpretation and entailment to \DatalogZ{}. An interpretation $\ipret$ is a set of facts. $\ipret$ satisfies a ground atom $\alpha$, $\ipret \models \alpha$ if one of the following holds (i) $\alpha$ is a standard atom and $\ipret$ contains each fact obtained from $\alpha$ by evaluating the numeric terms in $\alpha$, or (ii) $\alpha$ is a comparison atom evaluating to true under the usual semantics of comparisons. This notion can be extended to conjunctions of ground atoms, rules, and programs as in FO logic. A model is an interpretation that satisfies all rules of a program: $\ipret \models \lprog$. We say $\lprog$ entails a fact, $\lprog \models \alpha$, if whenever $\ipret \models \lprog$ then $\ipret \models \alpha$ holds. 
We are specifically interested in Fact Entailment (\factent{}) as a fundamental reasoning problem for Datalog languages. 
\begin{definition}
\factent{} is the problem of deciding whether a database $D$ and a set of rules $\Sigma$ entail a fact $\alpha$, i.e., $(D \cup \Sigma) = \lprog \models \alpha$.
\end{definition}

\noindent
Note that there are several related problems, such as (Boolean) Conjunctive Query Evaluation (\bcqeval{}). It is easy to show that these problems are equivalent.

\paragraph{Warded \DatalogPm{}.}
A Tuple Generating Dependency $\sigma$ is a first order sentence of the form
$\forall \mathbf{x} \forall \mathbf{y} (\phi(\mathbf{x},\mathbf{y}) \to \exists \mathbf{z} \psi(\mathbf{x},\mathbf{z})),$
where \textbf{x},\textbf{y},\textbf{z} are tuples of variables and $\phi, \psi$ are conjunctions of atoms without constants and $\mathit{nulls}$. The set of existential variables in $\sigma$ is \vare($\sigma$) $=$ \vars($\psi$) $\setminus$ \frontier($\sigma$). 
For brevity, we typically omit quantifiers and replace conjunction symbols with commas. A TGD $\sigma$ is \textit{applicable} with respect to an interpretation $I$ if there exists a homomorphism $h$ s.t. $h($\body$(\sigma)) \subseteq I$. An interpretation $I$ satisfies a TGD $\sigma$, denoted $I \models \sigma$, if the following holds: If there exists a homomorphism $h$ in $\sigma$ s.t.\ $h(\phi(\mathbf{x},\mathbf{y})) \subseteq I$ then there exists a homomorphism $h' \supseteq h_{|\mathbf{x}}$ s.t.\ $h'(\psi(\mathbf{x},\mathbf{z})) \subseteq I$.
An interpretation $I$ satisfies a set of TGDs $\Sigma$, $I \models \Sigma$, if $I \models \sigma$ for each $\sigma \in \Sigma$. 
Given a database $D$ and a set of TGDs $\Sigma$, an interpretation $I$ is a \textit{model} of $D$ and $\Sigma$ s.t. $I \supseteq D$ and $I \models \Sigma$. 

Warded \DatalogPm{}~\cite{DBLP:conf/pods/ArenasGP14,Gottlob:SpaceEfficientCore} is an extension of \DatalogPm{}, which is a family of Datalog languages with existentials in rule heads.
Warded \DatalogPm{} restricts the syntax of TGDs to control the use of so-called \textit{dangerous} variables that can be unified with $\mathit{null}$ values in existential rule heads during program evaluation. The goal is to restrict the derivation of predicates with $\mathit{null}$ values, since it was shown that unrestricted use of dangerous variables has a direct relation to high complexity of reasoning~\cite{DBLP:journals/jair/CaliGK13}.
To define wardedness, we need several preliminary definitions. 
We define the set of \textit{positions} of a schema $S$, \pos($S$) as the set $\set{P[i] \mid P/n \in S, i \in [1,n]}$. We write \pos($\Sigma$) to denote the respective set over a set of TGDs $\Sigma$. The set of existential variables occurring in a TGD $\sigma$ is defined as \vare($\sigma$) $=$ \vars($\psi$) $\setminus$ \frontier($\sigma$).
The set of \textit{affected positions} of sch($\Sigma$), denoted \aff($\Sigma$) is defined inductively: If there exists a rule $\sigma \in \Sigma$ and a variable $x \in$ \vare($\sigma$) at position $\pi$, then $\pi \in$ \aff($\Sigma$), and if there exists a rule $\sigma \in \Sigma$ and a variable $x \in $ \frontier($\sigma$) in the body of $\sigma$ only at positions of \aff($\Sigma$) and variable $x$ appears in the head of $\sigma$ at position $\pi$, then $\pi \in$ \aff($\Sigma$). The set of \textit{non-affected positions} is defined as: $\text{\nonaff}(\Sigma) = \text{\pos}(\Sigma) \setminus \text{\aff}(\Sigma)$. 
Depending on the positions variables occur at in rules, and how the variables ``interact'' with $\mathit{nulls}$ and other variables, we can classify them into three stages. For a TGD $\sigma \in \Sigma$ and a variable $x \in$ \body($\sigma$): 
    (i) $x$ is \textit{harmless} if at least one occurrence of it is in \body($\sigma$) at a position $\pi \in$ \nonaff($\Sigma$)
    (ii) $x$ is \textit{harmful} if it is not harmless
    (iii) $x$ is \textit{dangerous} if it is harmful and $x \in$ \frontier($\sigma$).
\begin{definition}
A set $\Sigma$ of TGDs is \textit{warded} if either 
    (i) for each $\sigma \in \Sigma$, there are no dangerous variables in \body($\sigma$), or
    (ii) there exists an atom $\alpha \in $ \body($\sigma$), called \textit{ward}, s.t.\ all the dangerous variables in \body($\sigma$) occur in $\alpha$ and each variable of \vars($\alpha$) $\cap$ \vars(\body($\sigma$) $\setminus$ $\set{\alpha}$) is harmless.
\end{definition}

\paragraph{Descriptive Complexity.} Let us introduce some basic definitions from the area of finite model theory. A vocabulary $\tau = \set{R_1,...,R_n,c_1,...c_n}$ is a finite set of relation symbols with specified arity and constant symbols. A $\tau$-structure is a tuple $\mathfrak{A}=(A, R_1^{\mathfrak{A}},...,R_n^{\mathfrak{A}},c_1^{\mathfrak{A}},...c_n^{\mathfrak{A}})$, s.t. $A$ is a nonempty set, called the universe of $\mathfrak{A}$. Each $R_{i}^{\mathfrak{A}}$ is a relation over $A$ and each $c_{j}^{\mathfrak{A}} \in A$. A finite $\tau$-structure has a finite universe. Observe that a relational database is a finite relational structure. A Boolean query $Q$ is a mapping $Q: \mathcal{K} \rightarrow \set{0,1}$ that is closed under isomorphism.

\section{Related Work}\label{sec:rel-work}
Limit \DatalogZ{} was introduced by Kaminski et al.~\cite{DBLP:conf/ijcai/KaminskiGKMH17,DBLP:conf/aaai/KaminskiGKH20} and allows for decidable arithmetic in Datalog over integers. The main idea is to introduce several semantic restrictions for \DatalogZ{} to obtain decidability. First, the authors introduce a special notion of predicates (limit predicates) which allow to represent interpretations finitely (since interpretations over integers are infinite in general). Decidability is obtained by essentially disallowing multiplication, i.e., allowing linear arithmetic terms only. Finally, an efficient (polynomial time) fragment is defined by further restricting the language through preventing divergence of numeric values during the application of the immediate consequence operator used for reasoning. The results of Kaminski et al.\ allow for tractable arithmetic in \DatalogZ{}. However, the approach uses a rather implicit semantic notation, which may be cumbersome to use. We try to tackle these problems by introducing a purely syntactical formulation in Section~\ref{ssec:synt-fragment}. Moreover, as mentioned above, limit \DatalogZ{} does not support existential rule heads and is therefore too weak for ontological reasoning or KG reasoning. Based on the FO notion of limit \DatalogZ{},~\cite{DBLP:conf/lics/BurnORW21} proposed a generalization to high-order logic and defined a decidable fragment (i.e., decidable satisfiability).

The formalism introduced by Ross and Sagiv (RS)~\cite{DBLP:conf/pods/RossS92} based on monotonicity is similar to limit \DatalogZ{}. Contrary to other solutions, it deals with all four \textit{common aggregates}, i.e., sum, count, min, and max at once in a coherent manner.
The RS formalism suffers from the fact that checking monotonicity is undecidable. Moreover, that monotonicity does not imply decidability of reasoning (of \factent{}) in general. As all other known arithmetic languages, the RS-formalism does not support rules with existentials in heads.

\cite{DBLP:journals/vldb/MazuranSZ13} proposed \DatalogFS{} which adds frequency support goals to the language, that allow to count the occurrences of atoms that satisfy these special goals. The most critical downside of the \DatalogFS{} formalism by Mazuran et al.\ is clearly its undecidability (of \factent{}). Furthermore, it does not allow existential rule heads and is thus too weak for KG reasoning tasks. The authors do not introduce a decidable fragment and do not provide theoretical complexity proofs of their general language. 

Arithmetic and its complexity also play a vital role in related logic programming formalisms, such as the non-monotonic paradigm answer set programming (ASP) and tasks related to these logics. Recently, Eiter and Kiesel have introduced a new logic called ASP with Algebraic Constraints~\cite{DBLP:journals/tplp/EiterK20}. Their logic is based on Weighted Logic and Here-and-There (HT) logic and introduces constraints on the values of weighted rules. They consider extensions such as aggregation and prove complexity results, for instance a \conpcls{}-completeness result for the model checking problem.
In another recent work Lifschitz~\cite{DBLP:journals/tplp/Lifschitz21} tackles the problem of proving equivalences between ASP programs with integer arithmetic and introduces a deductive system to prove such equivalences.

\section{Expressive Power of Limit Arithmetic}\label{sec:expr-power}
In this section we present results about the expressive power of limit \DatalogZ{}, which answers an open question posed by the line of work of Kaminski et al. The authors showed that reasoning in limit \DatalogZ{} is \conpcls{}-complete in data complexity. Hence, an interesting question now is: Knowing that reasoning with limit \DatalogZ{} is \conpcls{}-hard, what is the expressive power gained by this arithmetic extension, or more generally: How much expressive power does decidable arithmetic add to (plain) Datalog?
Informally put, from the point of view of descriptive complexity we know that every limit \DatalogZ{} definable query is in \conpcls{} and some such queries are even \conpcls{}-hard in data complexity by the \conpcls{}-completeness result of limit \DatalogZ{}. The question now is whether limit \DatalogZ{} is powerful enough to define every \conpcls{} computable query. Our main theorem of this section shows that when allowing a mild form of negation (semi-positive rules), reasoning with limit \DatalogZ{} is indeed as powerful as \conpcls{}-computation.
\newcommand{\thmone}{
Let $Q$ be a Boolean query. The following are equivalent:
\begin{itemize}
    \item[(i)] $Q$ is computable in \conpcls{}
    \item[(ii)] $Q$ is definable in limit \DatalogZ{}
\end{itemize}
That is, limit \DatalogZ{} = \conpcls{}.
}
\begin{theorem}\label{thm:desc-compl-main}
\thmone
\end{theorem}

\noindent
Recall that Fagin's seminal theorem states that the existential fragment of second order logic $\exists SO$ captures \npcls{} computation~\cite{fagin1974generalized}. Thus, together with Fagin's theorem, Theorem~\ref{thm:desc-compl-main} immediately implies the following Corollary, since the universal fragment of second-order logic, $\forall$SO, is the complement of $\exists$SO and therefore captures \conpcls{}. This corollary highlights the relationships between complexity, logic, and arithmetic in Datalog.

\newcommand{\corrone}{
Limit \DatalogZ{} = $\forall$SO = \text{\conpcls{}}.
}
\begin{corollary}\label{cor:limit-dz=uso}
\corrone
\end{corollary}

\section{Arithmetic in Warded \DatalogPm{}}
In this section we define our arithmetic extension of Warded \DatalogPm{} and discuss an efficient reasoning algorithm for our new language.

\paragraph{Negative Results.}
We prove that Warded \DatalogPm{} extended with integer arithmetic (and even with limit arithmetic) is undecidable in general. Hence, we cannot simply use arithmetic in Warded \DatalogPm{} and get decidability. Another point of view is that in general wardedness does not reduce the complexity of arithmetic in \DatalogPm{} per se. 

\newcommand{\thmthree}{
	\bcqeval{} for Warded \DatalogPm{} extended with integer arithmetic is undecidable.
}
\begin{theorem}\label{thm:negative-res}
\thmthree
\end{theorem}

\noindent
As corollary of this theorem note that our proof also holds for extensions of Warded \DatalogPm{} with limit arithmetic and for fragments like piece-wise linear Warded \DatalogPm{}~\cite{Gottlob:SpaceEfficientCore}.

\begin{corollary}
\bcqeval{} for (piece-wise) linear Warded \DatalogPm{} extended with limit arithmetic is undecidable.
\end{corollary}

\subsection{A Syntactic Arithmetic Fragment}\label{ssec:synt-fragment}
Before defining our Warded \DatalogPm{} extension we firstly consider a syntactic fragment of \DatalogZ{}, extending the limit notions of Kaminski et al.~\cite{DBLP:conf/ijcai/KaminskiGKMH17}. We show that we can exploit the techniques of limit arithmetic for our syntactic fragment to establish fundamental complexity results for reasoning tasks.

\paragraph{Syntax.}\label{def:bound-datalogz-syntax} We define Bound \DatalogZ{} as extension of \DatalogZ{} with a set of additional \textit{bound} operators: $\rho(x) \in \set{max(x), min(x)}$ where $x$ is a numeric term. Numeric predicates are either exact or bound numeric, where the last position contains a bound operator $\rho$. All exact numeric predicates are EDB. A rule $r$ is a \textit{Bound} \DatalogZ{} rule, if:
\begin{itemize}
    \item[(i)] \body($r$) = $\emptyset$, or 
    \item[(ii)] each atom in sb($r$) is object, exact numeric or a bound atom and \head($r$) is an object or a bound atom. 
\end{itemize}
In general, a Bound \DatalogZ{} rule is written as: $\varphi \rightarrow \psi(\mathbf{t}, \rho(a)).$ A Bound \DatalogZ{} program consists of Bound \DatalogZ{} rules only.

\paragraph{Semantics.} Intuitively, we use designated \textit{bound operators} to only store the upper or lower bounds of integer values of a predicate in interpretations and define that the other values are implied implicitly. That is, a bound operator $A(\mathbf{t},\rho(x))$ has the meaning that the value $x$ of $A(\mathbf{t}, x)$ is at least $x$ if $\rho$ is $max$ or at most $x$ if $\rho$ is min. However, we do not enforce the semantics for all facts containing predicate $A$.
For instance, fact $A(\mathbf{t}, min(3))$ says that $A$ for the tuple of objects $\mathbf{t}$ has value at most 3, as $A(\mathbf{t}, min(3))$ implies that the facts $A(\mathbf{t}, min(k))$ also hold for $k\geq 3$ in interpretations (see Example \ref{ex:bounded-datalogz}). Standard notions of limit \DatalogZ{}, such as interpretations, rule applicability of arithmetic rules, and linear arithmetic (essentially disallowing multiplication) can easily be extended to bound programs. 
It is straight forward to show that the complexity results carry over to our fragment. Note that this syntactical version can express all queries limit \DatalogZ{} can.
\begin{example}\label{ex:bounded-datalogz}
Bound \DatalogZ{} allows users to express the computation of the shortest paths from a dedicated start vertex with the following two rules (assuming a predicate $edge$ encoding the graph): 
\begin{align}
\mathit{Path}(start, min(0)). \\
\mathit{Path}(v, min(x)), \mathit{Edge}(v,w,min(y))  \nonumber \\
\rightarrow \mathit{Path}(w, min(x+y)).
\end{align}
The semantics follows the intuition that the existence of a shortest path to $v$ of cost at most $x$ $(\mathit{Path}(v, min(x)))$ implies the existence of a path for all $k \geq x$.
\end{example}

\noindent
Now that we have constructed a syntactic fragment capable of integer arithmetic, we have a good setup in terms of arithmetic in our language.
We now want to extend Warded \DatalogPm{} (which does a priori not allow arithmetic) with our syntactic bound fragment. Thus, we leverage techniques from Warded \DatalogPm{} to handle reasoning with existentials in rule heads and prove the termination of programs.

\paragraph{Warded Bound \DatalogZ{}.} 
In our language, object and numeric variables are being separated in the sense that predicates have arguments that are either all (tuples of) objects ($P(\mathbf{t})$), or objects with the last position being numeric ($P(\mathbf{t}, x)$). Object predicates may contain $\mathit{null}$ values and numeric predicates may be exact predicates (EDB) or contain arithmetic expressions, which are only allowed in bound operators $min$ or $max$.
\begin{definition}
Warded Bound \DatalogZ{} is Bound \DatalogZ{} extended with existentials in the rule heads such that the following conditions hold:
\begin{itemize}
    \item[(i)] Existentially quantified variables may only appear in object positions
    \item[(ii)] Each rule is warded (thus we extend wardedness by numeric variables and atoms accordingly)
\end{itemize}
\end{definition}

\noindent
Observe that we allow predicates containing both existentially quantified variables and arithmetic terms in rule heads, e.g.: $P(a, 3) \rightarrow \exists \nu P(\nu, 3)$.
As highlighted above, our language contains Warded \DatalogPm{} which directly implies that it is capable of efficient Knowledge Graph reasoning, thus it supports all queries Warded \DatalogPm{} can express. The expressive power of Warded \DatalogPm{} has been argued in detail in~\cite{Gottlob:SpaceEfficientCore}. The following example showcases the use of our language for advanced queries in KG reasoning systems.

\begin{example}{\textbf{(Family Ownership)}}\label{ex:warded-bound-datalogz_main_example}
Detecting connections between ownership relations of company shares owned by families is a main use case in financial KG applications~\cite{DBLP:conf/icde/BellomariniFGS19}. The relation $\mathit{family}$ is assumed to be incomplete. We use the following predicates: 
$\mathit{Person}(p, \mathbf{x})$: $p$ is a person with property vector $\mathbf{x}$, $\mathit{Family}(f, p)$: person $p$ belongs to family $f$, $\mathit{Asset}(a,c,v)$: an asset $a$ of a company $c$ whose (absolute) value is $v$, $\mathit{Right}(o,a,w)$: $o$ has right on a number $w$ of shares of asset $a$, and $\mathit{Own}(f,a,w)$: a family $f$ owns a number $w$ of shares of asset $a$.

\begin{align}
    \mathit{Person}(p, \mathbf{x}) \rightarrow& \exists f \mathit{Family}(f, p).& \label{running-ex:r1} \\
    \mathit{Person}(p,\mathbf{x}), Per&son(p',\mathbf{x}), &\nonumber \\
    \mathit{Family}(f,p) \rightarrow& \mathit{Family}(f,p').& \label{running-ex:r2}\\
    \mathit{Right}(o, a, w), As&set(a,c,v), \mathit{Right}(c, a', w')& \nonumber \\
    \rightarrow& \mathit{Right}(o,a', max(w + v + w')). &\label{running-ex:r3}\\
    \mathit{Right}(p,a,w), \mathit{Fa}&mily(f,p), \mathit{Own}(f,a,x)& \nonumber \\
    \rightarrow& \mathit{Own}(f,a,max(w+x)).&\label{running-ex:r4}
\end{align}
Rule~\ref{running-ex:r1} defines implicit knowledge that if a person with some $\mathbf{x}$ exists, then there exists a family relation which contains $p$. The pre-computed information in $\mathbf{x}$ is used to identify related people. In practice, this tasks could be delegated to other KG components (\cite{DBLP:conf/icde/BellomariniFGS19}). However, we use only language constructs for which we have complexity guarantees here. Rule~\ref{running-ex:r2} is used to check whether the families of two persons coincide. Rule~\ref{running-ex:r3} defines transitivity on the rights of an asset. Finally, Rule~\ref{running-ex:r4} expresses the total ownership of a family based on the rights of its members. 
\end{example}

\noindent
A simple consequence operator-like procedure for ordinary Bound \DatalogZ{} arithmetic may not terminate due to potential divergence of arithmetic terms. This would be the case, for instance, if a program contained the rule $A(max(x)) \rightarrow A(max(x+1))$. Thus, to use an iterative approach, we need to ensure that the arithmetic expressions may not diverge by restricting rules to be \textit{stable}, as proposed in~\cite{DBLP:conf/ijcai/KaminskiGKMH17}. Stability enforces a kind of mild monotonicity to the rules s.t.\ divergence can be avoided. However, checking stability is undecidable, hence we use the notion of type-consistent programs and assume that our programs are type-consistent, defined as follows.

\begin{definition}\label{def:type-consistency}
A Warded Bound \DatalogZ{} rule $\sigma$ is type-consistent if the following conditions hold:
\begin{enumerate}
    \item[(i)] Each numeric term is of the form $k_0 + \sum_{i=1}^{n} k_i \times m_i$ where $k_0$ is an integer and each $k_i$, $i \in [1,n]$ is a non-zero integer
    \item[(ii)] If \head($\sigma$) = $A(\mathbf{t}, \rho(\ell))$ then each variable in $\rho(\ell)$
    with a positive (negative) coefficient $k_i$ occurs also in a unique bound atom of $\sigma$ that is of the same (different) type, i.e., min or max, as $\rho$.
    \item[(iii)] For each comparison predicate $(t_1 < t_2)$ or $(t_1 \leq t_2)$ in $\sigma$, each variable in $t_1$ with positive (negative) coefficient also occurs in a unique min (max) operator in a body atom and each variable in $t_2$ with a positive (negative) coefficient also occurs in max (min) operator in a body atom of $\sigma$.
\end{enumerate}
\end{definition}

\noindent
Type-consistency introduces purely syntactical and local restrictions to the rules. Most notably, it ensures that the value in the head of a rule of a numeric variable $x$ behaves analogously to the value in the body occurrence of $x$. That is, if the head contains a variable $x$ inside a bound operator $\rho$, then $x$ also appears in the body in a bound operator of the same type as $\rho$ (min or max). It was shown that type-consistency implies stability and that type-consistency can be checked efficiently in \logspacecls{}~\cite{DBLP:conf/ijcai/KaminskiGKMH17}.

\subsection{Reasoning Algorithm}\label{ssec:e-exst}

\begin{algorithm}[t]
    \SetAlgoLined
    \SetAlgoNoEnd
    \SetKwInput{KwInput}{Input} 
	\KwInput{A Warded Bound \DatalogZ{} program $\lprog = D \cup \Sigma$, a fact $\alpha$}
	\KwResult{True if $\lprog \models \alpha$}
	$\pipret := \emptyset$\;
	\Repeat{$\pipret_{curr} = \pipret$}{
	    $\pipret := \pipret_{curr}$\;
    	$\mathit{update}(\mathcal{G}_{\lprog}^{\pipret})$\; \label{alg:line-update-vpg}
    	\ForEach{$\sigma \in \Sigma$}{
    	    \If{($\gamma$ = $\mathit{checkApplicable}(\sigma, \pipret_{curr}))$ != $\mathit{NULL}$}{\label{alg:line-check-exists-trigger} 
    	        \If{$\mathit{checkTermination}(\gamma)$}{\label{alg:line-check-termination}
    	            $\pipret_{curr} := \pipret_{curr} \cup \gamma$\;\label{alg:line-add-fact} 
    	        }
    	    }
    	}
	}\label{alg:line-check-interpret-match}
	\Return true if $\pipret \models \alpha$; \label{alg:line-model-check}
	\caption{Algorithm for Warded Bound \DatalogZ{}}\label{algo:chase-for-wbldz}
\end{algorithm}

We are ready to present the main reasoning algorithm, depicted in Algorithm~\ref{algo:chase-for-wbldz}. Note that our procedure is an extension of the general algorithm for Warded \DatalogPm{}, extending ideas from Bellomarini et al.~\cite[Algorithm 1]{Bellomarini:VadalogSystem}. In essence, this algorithm is a forward chaining procedure, which derives new facts from the given program step by step in an iterative way. When deriving a new fact $\gamma$, we must pay attention to two issues: (i) we must avoid divergence of numeric terms and (ii) we have to restrict the propagation of $\mathit{null}$ values for rules with existentials to avoid high complexity. 
With respect to issue (i) we apply the value propagation graph method proposed in~\cite{DBLP:conf/ijcai/KaminskiGKMH17} to keep track of the arithmetic computations. The value propagation graph $\mathcal{G}_{\lprog}^{\pipret}$ of a program w.r.t.\ an interpretation contains a node for each bound atom of the interpretation and contains edges between body and head atoms of a rule. The edge labels encode how numerical values are propagated from body to head atoms. Thus, the edge labels can be used to check if numerical terms diverge or not when a particular rule is applied. In each iteration of the main loop of our algorithm, we update the value propagation graph (adding new nodes and edges and updating edge weights). The update of the graph in Line \ref{alg:line-update-vpg} replaces the diverging arithmetic terms with the symbol $\infty$ in the current interpretation. 

After this step, we begin to derive new facts by applying the rules (similar to an immediate consequence computation). Thus, for each rule in the input program, we check if its applicable and if so, we try to derive a new fact from it. Because of issue (ii), this step requires an additional termination check since we cannot uncontrollably derive new $\mathit{nulls}$. Therefore, we use the termination check of wardedness (\textit{checkTermination} as proposed in~\cite{Bellomarini:VadalogSystem}) to check if we can safely derive the current fact. Intuitively, this termination procedure uses several guiding structures, such as warded forests, to detect isomorphisms between facts. We refer the reader to~\cite{Bellomarini:VadalogSystem} for more details on the termination procedure. This mechanism allows us to avoid derivation of superfluous facts. It takes a fact as input and checks, using the underlying data structures, if no isomorphism for the fact to derive exists in the current derivation graph of the program. If so, we add the newly derived fact to the interpretation $\pipret_{curr}$ that is currently being computed (Line \ref{alg:line-add-fact}). Finally, if no new facts have been derived, i.e., the interpretation that has been computed in the current iteration $\pipret_{curr}$ is equal to the previously computed one ($\pipret$), we terminate the reasoning procedure (Line \ref{alg:line-check-interpret-match}), do a single model check on the fact given as input, and return the result of the check in Line \ref{alg:line-model-check}.

\begin{algorithm}[t]
	\SetAlgoLined
    \SetAlgoNoEnd
	\SetKwInput{KwInput}{Input} 
	\KwInput{Rule $\sigma$, interpretation $\pipret$}
	\KwResult{True if $\sigma$ is applicable to $\pipret$}
	Construct linear integer constraint, $\mathcal{C(\sigma, \pipret)}$\;\label{alg2:line-constr-C}
	\If{$\mathcal{C(\sigma, \pipret)}$ has an integer solution}{
	    \If{head($\sigma$) is object or exact numeric}{
            $\gamma_{\sigma, \pipret}$ = $head(\sigma)$\;
        }
        \If{head($\sigma$) contains a bound operator $A(\mathbf{t}, \rho(a))$}{
            let $\mathit{opt}(\sigma, \pipret)$ be the optimal integer solution of $\mathcal{C(\sigma, \pipret)}$\;
        	$\gamma_{\sigma, \pipret}$ = $A(\mathbf{t}, \mathit{opt}(\sigma, \pipret))$\;
        }
        \If{\head($\sigma$) contains an existential, $\exists \mathbf{z} A(\mathbf{z, x})$}{
            $\gamma_{\sigma, \pipret}$ = $A(\nu, \mathbf{x})$ for a new $\mathit{null}$ $\nu$\;
	    }
	    return $\gamma_{\sigma, \pipret}$\;
	}
	\Else{
	    return NULL\;
	}
	\caption{Procedure to check and compute rule applicability}\label{alg:check-applicability}
\end{algorithm}

The subprocedure shown in Algorithm \ref{alg:check-applicability} (called in Line \ref{alg:line-check-exists-trigger}) represents a vital step in rule derivation, namely, rule application. Since bound arithmetic with min and max operators is involved, we cannot simply derive a new arithmetic value, but we need to compute the \emph{optimal} value to derive, otherwise it would be possible to derive inconsistent facts which violate the semantics of the min and max bound operators. In essence, the call to $\mathit{checkApplicability}$ constructs a linear integer constraint (Line \ref{alg2:line-constr-C}), which has a solution if a rule is applicable w.r.t.\ an interpretation, according to the semantics for Warded Bound \DatalogZ{}. Then, a rule can have the following properties according to the syntax of Warded Bound \DatalogZ{}: The current rule (i) may contain an existential head, and (ii) either contains an object or exact numeric head atom, or (iii) a bound atom in the head

For case (i) we need to enforce the semantics of existential rules, construct a fresh, unused $\mathit{null}$ in the object position in case an existential is present and derive an atom containing the new $\mathit{null}$. For case (ii), we simply check if the rule is applicable to the current interpretation and if so, we simply derive the head atom. For case (iii), we need to compute the optimal value of the linear integer constraint, since we are dealing with $\mathit{min}$ or $\mathit{max}$ bound operators and derive the optimal value. Where optimal means the smallest integer solution for $\rho = \mathit{min}$ and the largest integer solution for $\rho = \mathit{max}$. Finally,  The value computed for the current rule $\gamma_{\sigma, \pipret}$ is then returned to be made available to the main algorithm. If the rule is not applicable the subprocedure returns the empty value ``NULL'' to the main algorithm.

\newcommand{\thmtwo}{
For a Warded Bound \DatalogZ{} program $\lprog{}$ and a fact $\alpha$, deciding whether $\lprog \models \alpha$ is \pcls{}-complete in data complexity.
}
\begin{theorem}\label{thm:warded-bdz-ptime-algorithm}
\thmtwo
\end{theorem}

\noindent
Note that the bound in Theorem~\ref{thm:warded-bdz-ptime-algorithm} is tight since both reasoning with existentials and arithmetic in stable limit \DatalogZ{} programs ((plain) Datalog resp.) is \pcls-complete in data complexity. 
We would like to highlight that as extension of Warded \DatalogPm{} our language can express all queries Warded \DatalogPm{} can.

\section{Conclusion and Future Work}\label{sec:conclusion}
Our main goal was to provide the first complexity result for a \DatalogPm{} language that supports arithmetic. Motivated by finding an efficient and powerful arithmetic extension of Warded \DatalogPm{}, we closed the following gaps in current research:
\begin{itemize}
    \item We proved that the recently introduced limit \DatalogZ{} fragment captures \conpcls{}.
    \item We defined a syntactic fragment based on ideas of Kaminski et al.\ to improve the syntax and thus the usability of the language.
    \item We defined an extension of Warded \DatalogPm{} with integer arithmetic based on our syntactic fragment.
    \item We showed \pcls{}-completeness of our language and propose an efficient reasoning algorithm.
\end{itemize}
\smallskip
Our language, as an extension of Warded \DatalogPm{}, is powerful enough for complex reasoning tasks involving ontological reasoning and arithmetic as needed in KGs amongst other modern reasoning systems.

In future work, it would be desirable to also support count and sum aggregates directly in our language since this extension would allow us to include all common aggregates in one language. Moreover, a practical implementation of our reasoning language in a KG system would allow us to show empirical results for the efficiency of Warded Bound \DatalogZ{}.
\section* {Acknowledgements}
 This work was supported by the Vienna Science and Technology Fund (WWTF) grant VRG18-013, and the ``rAIson data'' Royal Society grant of Prof.\ Georg Gottlob.
 
\bibliographystyle{named}
\bibliography{main}

\clearpage
\appendix
\section{Additional Definitions}
Let us introduce further definitions in the area of database theory and logical reasoning required for the following proofs.

\subsection{Database-theoretic Notions}\label{sec:prelim-database-theory}
Let us introduce fundamental definitions specifically related to a database-theoretic context. Logic programming and database theory are deeply interconnected. For definitions we follow~\cite{Gottlob:SpaceEfficientCore} in this section.
We need the following notions: a \textit{schema} $S$ is a finite set of predicates. An instance $J$ over $S$ is a set of atoms that contains constants and labelled $\mathit{nulls}$ that are equivalent to new Skolem constants, not appearing in the set of terms yet. A \textit{database} over $S$ is a finite set of facts over $S$. Since queries can be seen as functions, mappings between sets of atoms are important in a database context. An \textit{atom homomorphism} between two sets of atoms $A,B$ is a substitution $h: A \mapsto B$ s.t. $h$ is the identity for constants and $P(t_1, \dots, t_n) \in A$ implies $h(P(t_1,\dots,t_n)) = P(h(t_1),\dots, h(t_n)) \in B$.

We consider conjunctive queries, which are equivalent to simple SELECT-FROM-WHERE queries in SQL and are of main interest in database theory.
A conjunctive query (CQ) over a schema $S$ is a rule of the form 
\[Q(\mathbf{x}) \leftarrow P_1(\mathbf{z_1}), \dots, P_n(\mathbf{z_n}),\]
where the predicate $Q$ is used in the head only, $P_1, \dots,P_n$ are atoms without $\mathit{nulls}$ and $\mathbf{x}\subseteq \bigcup z_i$ are the \textit{output variables}. A Boolean conjunctive query (BCQ) is a CQ where the head predicate $Q$ is of arity zero, hence contains no (free) variables in $\mathbf{x}$. The answer to a BCQ $q$ is $yes$ if the empty tuple () is an answer of $q$.

The \textit{evaluation} of a CQ $q(\mathbf{x})$ over an instance $J$ is the set of all tuples of constants induced by a homomorphism $h(\mathbf{x})$, where $h: \textrm{\atoms}(q) \mapsto J$.

We mostly consider an extension of traditional logic programs, which allow existential quantification in rule heads. These types of rules are called \textit{Tuple Generating Dependencies} in database theory.

\begin{definition}[Tuple Generating Dependencies, TGDs]
A Tuple Generating Dependency $\sigma$ is a first order sentence of the form
\[\forall \mathbf{x} \forall \mathbf{y} (\phi(\mathbf{x},\mathbf{y}) \to \exists \mathbf{z} \psi(\mathbf{x},\mathbf{z})),\]
where \textbf{x},\textbf{y},\textbf{z} are tuples of variables and $\phi, \psi$ are conjunctions of atoms without constants and $\mathit{nulls}$.
\end{definition}
For brevity, we typically write a TGD as $$\phi(\mathbf{x},\mathbf{y}) \to \exists \mathbf{z} \psi(\mathbf{x},\mathbf{z})$$ and use comma instead of $\land$ to indicate conjunctions (and joins respectively). A TGD $\sigma$ is \textit{applicable} w.r.t.\ an interpretation $I$ if there exists a homomorphism $h$ s.t. $h($\body$(\sigma)) \subseteq I$. 

Having the definition of rule applicability at hand, we can now define the formal model-theoretic semantics of TGDs w.r.t.\ an interpretation.

\begin{definition}[Semantics of TGDs]\label{def:sem-of-tgds}
An instance $J$ satisfies a TGD $\sigma$, denoted $J \models \sigma$, if the following holds: If $\sigma$ is applicable w.r.t.\ $J$, i.e., there is a homomorphism $h$ in $\sigma$ s.t. $h(\phi(\mathbf{x},\mathbf{y})) \subseteq J$ then there is a homomorphism $h' \supseteq h_{|\mathbf{x}}$ s.t. $h'(\psi(\mathbf{x},\mathbf{z})) \subseteq J$.
An instance $J$ satisfies a set of TGDs $\Sigma$, $J \models \Sigma$, if $J \models \sigma$ for each $\sigma \in \Sigma$.
\end{definition}

Since we are concerned with the complexity of reasoning with logical languages, we need to define fundamental reasoning tasks whose (worst-case) complexity we use as measure. These are the main problems we use in our complexity theoretic investigations. Conjunctive query answering is one of the fundamental reasoning tasks of TGDs~\cite{Gottlob:SpaceEfficientCore}. 

\begin{definition}[CQ answering under TGDs]
Given a database $D$ and a set of TGDs $\Sigma$, an instance $J$ is a \textit{model} of $D$ and $\Sigma$ s.t. $J \supseteq D$ and $J \models \Sigma$. We use $\text{mods}(D, \Sigma)$ to denote the set of all models of $D, \Sigma$. 
\end{definition}

One of the main tasks considered in database systems is computing the \textit{certain answers}, $\text{\certans}(q, D , \Sigma)$ to a query $q$, which is usually defined as the set of tuples occurring in all models of a database and a set of TGDs: $$\text{\certans}(q, D, \Sigma) = \bigcap_{J \in \text{mods}(D, \Sigma)} q(J).$$

For complexity theoretic questions, one often considers the corresponding decision problem of deciding whether for a tuple $\mathbf{c}$ it holds that $\mathbf{c} \in$ \certans $(q, D, \Sigma)$. We denote this decision problem as \cqans($\Sigma$).

\paragraph{Chase Procedure.}
The standard technique that can be applied to solve the problem of computing certain answers is the \textit{chase procedure}~\citeappendix{DBLP:conf/pods/JohnsonK82,DBLP:conf/sigmod/MaierMS79}.
This is a forward chaining approach to compute certain answers to a query. We let \chase($J, \Sigma$) denote the result of applying the chase procedure for an instance $J$ under a set of TGDs $\Sigma$ (i.e., a universal model which is homomorphically embeddable into every other model of $D$ and $\Sigma$). 
A standard result in database theory is that given a database $D$ and a CQ $q$, $\text{\certans}(q, D, \Sigma)$ = $q(\text{\chase}(D, \Sigma))$.

A central problem in database theory is the problem of checking whether a certain tuple can be derived from a database and a set of rules, i.e., whether the database and the set of rules entail a certain tuple. This problem is usually called query evaluation.

\begin{definition}[BCQ Evaluation Problem, \bcqeval]
Given a BCQ $q$, a set of TGDs $\Sigma$ and a database $D$, the BCQ evaluation problem is to decide whether the empty tuple () is entailed:
\[ D \cup \Sigma \models q\]
\end{definition}

Note that showing that $D \cup \Sigma \models q$ is equivalent to showing that $D \cup \Sigma \cup \set{\lnot q}$ is an unsatisfiable theory (i.e., an unsatisfiable set of first-order sentences)~\cite{DBLP:journals/jair/CaliGK13}. A classical result in database theory shows that two of the most fundamental problems, \cqans{} and \bcqeval{} are \logspacecls{}-equivalent~\cite{DBLP:journals/jair/CaliGK13}.

Another fundamental problem that is often used to prove complexity results about Datalog languages is the problem of deciding whether a database $D$ and a set of rules $\Sigma$ entail a certain fact.
\begin{definition}[Fact Entailment Problem, \factent{}]
Given a program $\lprog = D \cup \Sigma$ and a fact $\alpha$, the problem of deciding whether $\lprog{} \models \alpha$ is denoted \factent{}.
\end{definition} 
Note that since we can encode a CQ as a rule (a CQ is essentially a rule), the notion of fact entailment subsumes conjunctive query answering~\citeappendix{DBLP:conf/rweb/Kostylev20,DBLP:journals/ai/BagetLMS11}. 
The following easily verifiable proposition formalizes this intuition.

\begin{proposition} 
  The problem \bcqeval{} is polynomial-time (Karp) reducible to \factent{}.
\end{proposition}
\begin{proof}
Let $Q = \phi(\mathbf{x}) \rightarrow q$ be a BCQ. We construct an instance of \factent{} $(P, \alpha)$: $P = \phi(\mathbf{x}) \rightarrow q \text{ and } \alpha = q$.
If $Q$ is a positive instance of \bcqeval{} then there exists a homomorphism $h: h(\phi(\mathbf{x}))\in I$ for an interpretation $I$. Then $I \models P$ implies $I \models \alpha$, thus $P \models \alpha$. 
If $P \models \alpha$ then $I \models \alpha$ whenever $I \models P$. Hence, $I\models \phi(\mathbf{x})$ and $q \in I$.
The reduction is clearly computable in polynomial time.
\end{proof}

As with the problems above, we consider both, the data complexity and the combined complexity of \factent{}: 

\begin{itemize}
    \item Combined complexity: we assume $\lprog$ and $\alpha$ are part of the input 
    \item Data complexity: we assume $\lprog$ to be given as a dataset $\dset$ and a program $\lprog'$: $\lprog = \lprog' \cup \dset$ where only $\dset$ and $\alpha$ are part of the input and $\lprog'$ is fixed. 
\end{itemize}

\section{Proofs for Section 4}
Note that from the definition of a Boolean query, it is easy to see that this definition equivalent to stating that $Q$ coincides with the subclass 
$$\mathcal{K'} \subseteq \mathcal{K}, \mathcal{K'}=\set{\mathfrak{A} \in \mathcal{K} : Q(\mathfrak{A})=1}.$$

Due to this formulation, a Boolean query is also said to be a \textit{property} of $\mathcal{K}$.
\begin{definition}[Query]
Let $\tau$ be a vocabulary and $k$ a positive integer.
\begin{itemize}
    \item A k-ary query $Q$ on a class $\mathcal{K}$ of $\tau$-structures is a mapping with domain $\mathcal{K}$ s.t.
    \begin{itemize}
        \item For $\mathfrak{A} \in \mathcal{K}$ $Q(\mathfrak{A})$ is a k-ary relation on $\mathfrak{A}$, and
        \item Q is closed (preserved) under isomorphism;
    \end{itemize}
    \item A Boolean query $Q$ is a mapping $Q: \mathcal{K} \mapsto \set{0,1}$ that is closed under isomorphism. This is equivalent to stating that $Q$ coincides with the subclass $\mathcal{K'} \subseteq \mathcal{K}, \mathcal{K'}=\set{\mathfrak{A} \in \mathcal{K} : Q(\mathfrak{A})=1}$. Due to the latter formulation a Boolean query is often said to be a \textit{property} of $\mathcal{K}$.
\end{itemize}
\end{definition}

\begin{definition}[L-definable Query]
Let $L$ be a logic and $\mathcal{K}$ a class of $\tau$-structures. A Boolean Query $Q$ on $\mathcal{K}$ is $L$-definable if there exists an $L$ formula $\varphi(\mathbf{x})$ with free variables $\mathbf{x}$ s.t. for every $\mathfrak{A} \in \mathcal{K}$:
$$Q(\mathfrak{A}) = \set{(\mathbf{x}) \in A^k : \mathfrak{A} \models \varphi(\mathbf{a})}.$$

\end{definition}

It is important to say that the notion of $L$-definability on a class $\mathcal{K}$ of structures is a \textit{uniform} definability notion. Hence, the same $L$-formula serves as specification of the query on \textit{every} structure in $\mathcal{K}$. This is analogous to the concept of uniform computation of Turing machines known from computational complexity theory~\cite{DBLP:books/daglib/0023084}. \\
To show that a query $Q$ on $\mathcal{K}$ is $L$-definable, it suffices to construct a $L$-formula defining $Q$ on every structure in $\mathcal{K}$. Since our languages are in a Datalog context, let us define the notion of a Datalog query and its complexity.

\begin{definition}[Datalog Query]
A Datalog query is a pair $(\Pi, R)$, of a Datalog program $\Pi$ and a head predicate $R \in \Pi$. The query $(\Pi, R)$ associates the result $(\Pi, R)^{\mathfrak{A}}$ which is the interpretation of $R$ computed by $\Pi$ from the input $\mathfrak{A}$.
\end{definition}
A Datalog query essentially consists of a Datalog program and some designated predicate occurring in the program. The result of a query for a certain input structure is obtained by applying the fixpoint semantics of the Datalog language of the query. As in the area of logic programming, the data complexity of Datalog (or logics in general) plays a vital role for complexity investigations.

\begin{definition}[Data complexity of Logics]
Let $L$ be a logic. The data complexity of $L$ is the family of decision problems $Q_\varphi$ for each fixed $L$ sentence $\varphi$: Given a finite structure $\mathfrak{A}$, does $\mathfrak{A} \models \varphi$?
\end{definition}

We next want to define how we can determine whether the data complexity of a logic is in a complexity class $\mathcal{C}$, or even hard for a complexity class $\mathcal{C}$. The following is a standard definition of these two notions.

\begin{definition}[Complexity of Logics]
Let $L$ be a logic and $\mathcal{C}$ a complexity class.
\begin{itemize}
    \item The data complexity of $L$ is in $\mathcal{C}$ if for each $L$ sentence $\varphi$, the decision problem $Q_\varphi$ is in $\mathcal{C}$.
    \item The data complexity of $L$ is complete for $\mathcal{C}$ if it is in $\mathcal{C}$ and at there exists an $L$ sentence $\psi$ s.t. the decision problem $Q_\psi$ is $\mathcal{C}$-complete
\end{itemize}
\end{definition}

Informally put, from the point of view of descriptive complexity we know that every limit \DatalogZ{} definable query is in \conpcls{} and some such queries are even \conpcls{}-hard in data complexity by the \conpcls{} completeness result of positive limit \DatalogZ{}. The question now is if \DatalogZ{} is powerful enough to define every \conpcls{} computable query. Expressive power of logics is defined via the notion of \textit{capturing} complexity classes.

\begin{definition}[Logics Capturing Complexity Classes]
A language $\mathcal{L}$ captures a complexity class $\mathcal{C}$ on a class of databases $\mathcal{D}$ if for each query $Q$ on $\mathcal{D}$ it holds that: $Q$ is definable in $\mathcal{L}$ if and only if it is computable in $\mathcal{C}$.
\end{definition}

A cornerstone result in the development of descriptive complexity theory was the following seminal theorem by Fagin showing that the fragment of second-order logic that allows only existential quantification of relation symbols, existential second-order logic ($\exists$SO, sometimes denoted as $\Sigma_1^1$), captures non-deterministic polynomial time on the class of all finite structures.

\begin{theorem}[Fagin's Theorem~\cite{fagin1974generalized}]\label{thm:fagins-theorem}
Let $\mathcal{K}$ be an isomorphism-closed class of finite structures of some finite nonempty vocabulary. Then $\mathcal{K}$ is in \npcls{} if and only if $\mathcal{K}$ is definable by an existential second-order sentence.
\end{theorem}

Fagin's Theorem is a remarkable result, proving the first precise connection between a computational complexity class and logic. Informally, it ensures that a property of finite structures is recognizable in non-deterministic polynomial time if it is definable in $\exists$SO. Note that this immediately implies that the universal fragment of second-order logic ($\forall$SO) captures \conpcls{} by complementation.

\begin{definition}
Let $L$ be a logic and $\mathcal{K}$ a class of $\tau$-structures. A Boolean Query $Q$ on $\mathcal{K}$ is $L$-definable if there exists an $L$ formula $\varphi(\mathbf{x})$ with free variables $\mathbf{x}$ s.t. for every $\mathfrak{A} \in \mathcal{K}$:
$$Q(\mathfrak{A}) = \set{(\mathbf{x}) \in A^k : \mathfrak{A} \models \varphi(\mathbf{a})}.$$
\end{definition}

\noindent
Analogously, a Datalog query consists of a Datalog program $\lprog$ and some designated fact $\alpha$ occurring in the program. The evaluation of a query on a structure $\mathfrak{A}$ is the value of $\alpha$ with the relations of $\mathfrak{A}$ acting as extensional database (EDB) of $\lprog$. Usually, a Datalog program is defined by its declarative semantics given by a least fixpoint operator (as is the case for limit \DatalogZ{}). Expressive power of logics is defined via the notion of \textit{capturing} complexity classes.

\begin{definition}
A language $\mathcal{L}$ captures a complexity class $\mathcal{C}$ on a class of databases $\mathcal{D}$ if for each query $Q$ on $\mathcal{D}$ it holds that: $Q$ is definable in $\mathcal{L}$ if and only if it is computable in $\mathcal{C}$.
\end{definition}
Let us now rigorously state and prove our main theorem of Section \ref{sec:expr-power}.

\medskip\noindent
\textbf{Theorem \ref{thm:desc-compl-main}.} \textit{\thmone}

\smallskip
Note that later we consider semi-positive limit \DatalogZ{} that is, we allow negation of EDB (i.e., input) predicates in limit \DatalogZ{}. This is analogous e.g.\ to the capture result by Blass, Gurevich~\citeappendix{DBLP:conf/birthday/BlassG87}, and independently by Papadimitriou~\citeappendix{papadimitriou1985note} that Datalog captures \pcls{} on successor structures. Datalog alone is too weak to capture \pcls{}, hence in order to capture \pcls{} computation, negation of input predicates is needed (we discuss this issue in more detail in the proof). We begin with the simple direction, proving that the second item of Theorem~\ref{thm:desc-compl-main} implies the first, which is straightforward knowing that limit \DatalogZ{} is \conpcls{} complete.
\begin{lemma}
The limit \DatalogZ{} definable Boolean queries are all computable in \conpcls{}.
\end{lemma}

\begin{proof}
We can clearly compute every positive limit \DatalogZ{} query in \conpcls{}, i.e., the proposition follows from the \conpcls{}-completeness of limit \DatalogZ{}.
\end{proof}

It remains to show that every Boolean query on a class of finite structures computable in \conpcls{} is definable in positive limit \DatalogZ{}, i.e., that the properties decidable in \conpcls{} on finite structures are definable by positive limit \DatalogZ{} queries.

\begin{lemma}
All Boolean queries computable in \conpcls{} are definable by a Boolean limit \DatalogZ{} query. 
\end{lemma}
\begin{proof}
We show that every class of structures $\mathcal{K}$ that is recognizable by a \conpcls{} Turing machine is definable by a positive limit \DatalogZ{} query. To this end, we construct a positive limit \DatalogZ{} query $Q$ over a dataset $\mathcal{D}$ defining an input structure $\mathfrak{A}$. 

Before giving the proof, we want to point out several (standard but non-trivial) technicalities that originate mainly from the discrepancy between logic and our model of computation. 
\begin{itemize}
    \item Turing machines compute on string encodings of input structures, which implicitly provides a linear order on $\mathfrak{A}$ (the elements of the universe of $\mathfrak{A}$). For our proof this is not a problem since we can simply non-deterministically guess a respective linear order. We thus say a  TM $M$ decides a class of structures $\mathcal{K}$ if $M$ decides the set of encodings of structures in $\mathcal{K}$. From now on we fix some kind of canonical encoding (enabled through the linear ordering). We use $enc(\mathfrak{A},<)$ to refer to the set of encodings of a unordered structure $\mathfrak{A}$ having $<$ as linear order on the universe of $\mathfrak{A}$.
     \item Turing machines can consider each input bit separately, but Datalog programs cannot detect that some atom is not part of the input structure. This is due to the fact that negative information is handled via the closed world assumption and the fact that we only represent positive information in databases~\citeappendix{DBLP:journals/csur/DantsinEGV01}. Thus, we need to slightly extend the syntax of the programs we consider by allowing negated EDB predicates in rule bodies. In our proof this is only relevant for the input encoding. 
\end{itemize}

Consider a single-tape \conpcls{} TM $M'$ s.t. $M'$ recognizes a class of structures $\mathcal{K}$, i.e., $M'$ accepts $enc(\mathfrak{A},<)$ if $\mathfrak{A} \in \mathcal{K}$. 
Let $M = (\Gamma, S, \delta)$ be a single tape, non-deterministic polynomial time TM recognizing the complement of $\mathcal{K}$. That is, $M$ accepts $enc(\mathfrak{A},<)$ if $\mathfrak{A} \notin \mathcal{K}$ and rejects $enc(\mathfrak{A},<)$ if $\mathfrak{A} \in \mathcal{K}$. Let $n$ be the cardinality of the input structure of $M$. We assume w.l.o.g.\ that $M$ halts in at most $n^k$ steps (for some constant $k>0$) and that all computation paths of $M$ end in a halting state.

We represent the non-deterministic guesses of $M$'s non-deterministic transition function as binary string over the alphabet $\binlang$. By our assumption that $M$ halts in at most $n^k$ steps, a guess string corresponding to a computation path, $\pi$, in the configuration graph of $M$ represents an integer $i \in [0,n^k]$. We use $g(\pi)$ to denote this representation by adding 1 as the most significant bit in order to ensure each number $g(\pi)$ encodes a unique guess string. This mechanism is the key to encoding non-determinism in limit \DatalogZ{}~\citeappendix{DBLP:conf/rweb/Kostylev20}. A 0 guess at a certain configuration along a path $\pi$ in the configuration graph of $M$ can be represented by doubling $g(\pi)$, i.e., $g(\pi') = g(\pi) \cdot 2$ and for a 1 guess we double the number and add 1: $g(\pi') = g(\pi) \cdot 2 + 1$. We use $|\pi|$ to denote the length of $\pi$, i.e., the number of configurations in the computation path $\pi$ of the configuration graph of $M$.

\paragraph{Encoding, $\Pi_M$.}
We construct a positive limit \DatalogZ{} query $Q = (\Pi_M, \alpha)$ such that $Q$ evaluates to true on input $enc(\mathfrak{A},<)$ if and only if $\mathfrak{A} \notin \mathcal{K}$. The limit \DatalogZ{} program $\Pi_M$ consists of the following sets of rules:
\begin{itemize}
    \item $\Pi_{succ}$ computing the successor relation on $\mathfrak{A}$,
    \item $\Pi_{input}$: rules describing the input and ensuring that the encoding of $\mathfrak{A}$ is correct,
    \item $\Pi_{enc}$: rules that describe configurations,
    \item $\Pi_{comp}$ rules that enforce computation of $M$, and
    \item $\Pi_{rej}$ rules that check rejection, i.e., if \textit{all} computation paths lead to a rejecting state.
\end{itemize}

\paragraph{Successor Relation, $\Pi_{succ}$.}
As discussed above, we assume the existence of a successor relation $<$ on $\mathfrak{A}$. We let $succ$ be the object predicate encoding this relation and let $succ^+$ indicate the transitive closure of $succ$. We use EDB predicates $first$, $next$, and $last$ to encode the relation accordingly.

\paragraph{Input Encoding, $\Pi_{input}$.}
Note that it is not enough to just explicitly list the atomic facts defining the input configuration of $M$ for a given input string $u$. Hence, we encode successor structures (such as our input) s.t.\ there exist a quantifier-free formula $\beta_u(\mathbf{y})$ s.t.\ $\mathfrak{A} \models \beta_u(\mathbf{a}) \iff$ the $\mathbf{a}$-th symbol of the input configuration of $M$ for input $enc(\mathfrak{A}, <)$ is $u$. Let $\Pi_{input}$ be the Datalog program equivalent to $\beta_u(\mathbf{y})$.

\paragraph{Configuration Encoding, $\Pi_{enc}$.}
Since we assume that $M$ halts after $m = n^k$ steps and therefore uses at most $m$ tape cells, we can encode a configuration of $M$ as a $k$-tuple of objects from the universe of $\mathfrak{A}$. We encode a configuration $C$ with the following $(2k+1)$-ary max limit predicates. Intuitively, these predicates depend on arguments encoding space (tape cells), time (an integer corresponding to a configuration), and a guess string $g(\pi)$ for a run $\pi$ of $M$:
\begin{itemize}
    \item For each $q \in S$: $head_q(\mathbf{t, s}, g(\pi))$. For $k$-tuples over the universe of $\mathfrak{A}$, $\mathbf{t}$ encodes $|\pi|-1$ and $\mathbf{s}$ encodes the head position;
    \item For each $u \in \Gamma$: $tape_u(\mathbf{t,s}, g(\pi))$. For each $i \in [0,m-1]$ where $u$ is the symbol in the $i$-th tape cell in configuration $C$, $\mathbf{s}$ is a $k$-tuple encoding $i$ and $\mathbf{t}$ is as above 
\end{itemize}

\paragraph{Computation, $\Pi_{comp}$.}
To initialize the computation (the initial configuration), we start by encoding the input on the input tape, fill the rest of the tape with blank symbols, encode the head positioned in the left-most position, and set the initial state $q_{init}$ as the current state. The tuple $\mathbf{z_0}$ denotes a tuple consisting of $k$ values $z_0$ and $\mathbf{z_0'}$ is a tuple consisting of $k-1$ repetitions of $z_0$ used for the initial configuration. Rule~\ref{ln:head-init} initializes the head in the left-most position, initializing $head$ with $k$-tuples with the first symbol from the dataset $\mathcal{D}$. Similarly, Rule~\ref{ln:tape_init} encodes the input symbols from the input tape as $tape_u$ for $u \in \set{0,1}$ and Rule~\ref{ln:tape_blank} fills the rest of the tape with blank symbol $\square$. The last argument of all initialized predicates is set to $1 = g(\pi)$ to encode the first, single, initial configuration of $M$'s configuration graph. 
\begin{align}
    \mathit{first}(z_0) & \rightarrow \mathit{head}_{q_{init}}(\mathbf{z_0;z_0}; 1) \label{ln:head-init}\\
    \mathit{first}(z_0), \mathit{input}_u(x) & \rightarrow  \mathit{tape}_u(\mathbf{z_0;z'_0},x;1) \label{ln:tape_init}\\
    \mathit{first}(z_0), \mathit{last}(z_{max}),& \nonumber\\
    succ^+(\mathbf{z'_0}, z_{max}, z_{max}; \mathbf{x}) &\rightarrow  \mathit{tape}_{\square}(\mathbf{z_0;x};1) \label{ln:tape_blank}
\end{align}
We add the following rules for each $q \in S \setminus S^{acc} \cup S^{rej}$, and each $u,v \in \Gamma$ in order to represent an application of the transition function to the current configuration. Depending on the movement of the head, as indicated by the transition function of $M$, we let $succ'$ denote $succ(\mathbf{x'}, \mathbf{x})$ if the head move is $\hleft{}$ or $succ(\mathbf{x}, \mathbf{x'})$ if the head move is $\hright{}$. The atom $\mathit{differ}(\mathbf{x,y})$ is used to ensure that variables $\mathbf{y}$ are different from the head position encoded by $\mathbf{x}$. The rules encode a computation step in the configuration graph of $M$ when 0 is the current guess made by the transition function where the head moves from position $\mathbf{x}$ to position $\mathbf{x'}$, leaving unchanged positions $\mathbf{y}$ intact. In order to model the 0 guess we double $g(\pi)$ by computing $m'=m+m$.
\begin{align}
    \mathit{head}_q(\mathbf{t;x};m), \mathit{tape}_u(\mathbf{t;x};m), \mathit{tape}_u(\mathbf{t,y};m),\nonumber\\
    \mathit{succ}(\mathbf{t,t'}), \mathit{succ'}(\mathbf{x;x'}), \mathit{differ}(\mathbf{x;y}), (m + m = m') \rightarrow \nonumber\\
    \mathit{head}_{q'}(\mathbf{t';x'};m'), \mathit{tape}_{u'}(\mathbf{t';x};m'), \mathit{tape}_v(\mathbf{t';y};m') 
\end{align}

Analogously, we add the following rules for each $q \in S \setminus S^{acc} \cup S^{rej}$ and each $u,v \in \Gamma$ when the current guess is 1. We model a 1 guess as explained above by computing $g(\pi') = g(\pi) \cdot 2 + 1 = m + m + 1$.
\begin{align}
    \mathit{head}_q(\mathbf{t;x};m), \mathit{tape}_u(\mathbf{t;x};m), \mathit{tape}_u(\mathbf{t,y};m), & \nonumber\\
    \mathit{succ}(\mathbf{t,t'}), \mathit{succ'}(\mathbf{x;x'}), \mathit{differ}(\mathbf{x;y}), & \nonumber \\
    (m + m + 1 = m') & \rightarrow \nonumber\\
    \mathit{head}_{q'}(\mathbf{t';x'};m'), \mathit{tape}_{u'}(\mathbf{t';x};m'), \mathit{tape}_v(\mathbf{t';y};m') 
\end{align}

\paragraph{Rejection Check, $\Pi_{rej}$.}
For each rejecting state $r \in S^{reject}$, we derive the $\mathit{reject}$ max limit predicate keeping track of the configuration and corresponding guess string for rejecting configurations. We then propagate this information in a backtracking manner in the configuration graph of $M$'s computation depending on the guess strings for each configuration represented by the variables $m$ and $m'$.
\begin{align}
    \mathit{head}_r(\mathbf{t;x};m) &\rightarrow \mathit{reject}(\mathbf{t};m)\\
    \mathit{reject}(\mathbf{t'};m'), \mathit{succ}(\mathbf{t;t'}),& \nonumber \\ 
    (m + m = m') &\rightarrow \mathit{reject}(\mathbf{t};m)\\
    \mathit{reject}(\mathbf{t'};m'), \mathit{succ}(\mathbf{t;t'}), & \nonumber \\
    (m + m + 1 = m') & \rightarrow  \mathit{reject}(\mathbf{t};m)
\end{align}

Because of the backpropagation we simply need to check if for the initial configuration the $\mathit{reject}$ predicate is derived. It is easy to see that fact $\mathit{confirm}$ is derived if and only if all configurations reach a rejecting state.
\begin{align}
    \mathit{first}(z_0), \mathit{reject}(\mathbf{z_0};1) \rightarrow \mathit{confirm}
\end{align}

The following easily verifiable claims argue the correctness of the encoding above and follow the construction of $\Pi_M$.
\begin{claim}
If $M$ rejects $enc(\mathfrak{A}, <)$ then $(\Pi_{M}, \mathit{confirm})$ evaluates to true on $(\mathfrak{A}, <)$.
\end{claim}
Assume $M$ rejects $enc(\mathfrak{A}, <)$. Per definition of a \npcls{} TM, each computation path in the configuration graph of $M$ leads to a rejecting state. By construction of $\Pi_M$, for each configuration along a path $\pi$ in the configuration graph of $M$, at time $\mathbf{t}$, the atom $reject(\mathbf{t}, m)$ where $m  = |\pi|-1$ is derived. Thus, also for the initial configuration the fact $\mathit{confirm}$ is derived.

\begin{claim}
If $(\Pi_{M}, \mathit{confirm})$ evaluates to true on $(\mathfrak{A},<)$, then $M$ rejects $enc(\mathfrak{A}, <)$.
\end{claim}
Assume $(\Pi_M, \mathit{confirm})$ evaluates to true on $(\mathfrak{A}, <)$. Then for each rejecting state $r$ and each configuration along a computation path $\pi$ of $M$ leading to $r$, the atom $\mathit{reject}(\mathbf{t};m)$ is derived. But then we can ``reconstruct'' the configuration graph of $M$ from the predicates derived by $\Pi_M$. Since $\mathit{confirm}$ was also derived, we can conclude the input in the initial state led to a rejection, so $M$ clearly rejects.
\end{proof}

Together with Fagin's Theorem, Theorem~\ref{thm:desc-compl-main} immediately implies the following Corollary, since the universal fragment of second-order logic, $\forall$SO, is the complement of $\exists$SO and therefore captures \conpcls{}.

\medskip\noindent
\textbf{Corrolary \ref{cor:limit-dz=uso}}
\corrone
\medskip

Let us make several brief remarks on the implications and related work of Corollary~\ref{cor:limit-dz=uso} around universal SO-logic. Least fixpoint logic (LFP) is of great interest in descriptive complexity. It is well known that LFP cannot express every \pcls{}-computable query on finite structures. The Immerman-Vardi Theorem shows that, however, LFP can express all \pcls{}-computable queries on classes of finite ordered structures~\citeappendix{DBLP:conf/focs/Immerman80,DBLP:journals/iandc/Immerman86}. A well-studied fragment of LFP is LFP$_1$ which is an extension of FO logic with least fixpoints of positive formulae without parameters and closure under conjunctions, disjunctions, and existential and universal FO quantification (for a formal discussion we refer the reader to~\citeappendix{kolaitis2007expressive}). This language is also referred to as LFP(FO), i.e., least fixpoints of FO formulae by Immerman in~\citeappendix{DBLP:books/daglib/0095988}. It is well-known that the data complexity of LFP$_1$ is \pcls{}-complete. The expressive power of LFP$_1$ was studied by Kleene and Spector on the class of arithmetic $\mathfrak{N} = (\mathbb{N}, +, \times)$. The seminal \textit{Kleene-Spector Theorem}~\citeappendix{kleene1955arithmetical,spector1969inductively} shows that a relation $R \subseteq \mathbb{N}^k$ is LFP$_1$ definable on $\mathfrak{N}$ if and only if it is definable by a universal second order formula on $\mathfrak{N}$.

The Kleene-Spector Theorem establishes the equivalence LFP$_1(\mathfrak{N}) = \forall$SO($\mathfrak{N}$). Later this result was extended to countable structures $\mathfrak{A}$ that have a so-called FO coding machinery~\citeappendix{moschovakis2014elementary}, i.e., countable structures in which finite sequences of each length can be encoded by elements and decoded in a FO definable way~\citeappendix{kolaitis2007expressive}.

The language $\forall$SO was conjectured to properly include \nlogspacecls{}~\citeappendix{weinstein2007unifying}. If this conjecture holds then our result readily implies that positive limit \DatalogZ{} properly includes \nlogspacecls{}.

\section{Proofs for Section 5}
\medskip\noindent
\textbf{Theorem \ref{thm:negative-res}}
\thmthree
\begin{proof}
Similar to the proof of Theorem 10 in~\cite{DBLP:conf/ijcai/KaminskiGKMH17} we use a reduction from Hilbert's tenth problem~\citeappendix{hilbert1902mathematical}. This problem is to decide if a given Diophantine equation has an integer solution, i.e., solving $P(x_1, ... x_n) = 0$, for a multivariate polynomial $P(x_1,...,x_n)$ with integer coefficients. It is well known that the problem is undecidable even if the solutions are restricted to be non-negative integers, we use this variant of the problem. 

\smallskip\noindent
For every such polynomial $P$, let $D = \set{A(0)}$ and  $\Sigma$ contains the rules: 

\begin{align}
  A(x) \rightarrow & A(x+1) \\
  \bigwedge_{i=1}^{n} A(x_i) \land P(x_1,...,x_n) \doteq 0 \rightarrow & B()  
\end{align}

And let $q = B() \rightarrow Q$. Then 
$D\cup \Sigma \models q$ if and only if $P(x_1,...,x_n) = 0$ has a non-negative integer solution. Thus, deciding whether $D \cup \Sigma \models q$ is undecidable. 
\end{proof}

\begin{theorem}
For a Warded Bound \DatalogZ{} program $\lprog{}$ and a fact $\alpha$, deciding whether $\lprog \models \alpha$ is \pcls{}-complete in data complexity.
\end{theorem}
\begin{proof}
The main aspects for the complexity analysis of our main reasoning algorithm, Algorithm 1, are:
\begin{enumerate}
    \item Line 4 
    which updates the value propagation graph of the program w.r.t.\ the current interpretation.
    \item Line 6
    which involves constructing and solving an IP if the body contains numeric atoms, according to the definition of rule applicability of Bound \DatalogZ{} programs;
    \item The call in Line 7
    which uses the termination-check procedure for the Vadalog language proposed in~\cite{Bellomarini:VadalogSystem}.
\end{enumerate}
Updating the value propagation graph includes checking whether a node is on a positive weighted cycle in the propagation graph. This check (and the corresponding weight update) can be done, for instance, with a variant of the well-known Floyd-Warshall algorithm in \pcls{}-time. Constructing and solving an integer program for the applicability check can be done in polynomial time in data complexity since the integer programs have a ﬁxed number of variables (\cite{DBLP:conf/ijcai/KaminskiGKMH17}). Finally, the check-termination procedure requires storing several guiding structures and checking for isomorphism between facts. As shown in~\cite{Bellomarini:VadalogSystem}, for warded existential rules this check can clearly be done in polynomial time. \qedhere
\end{proof}

\section{Warded Bound \DatalogZ{} Examples}
Let us give several further reasoning examples expressed in our language to showcase the expressive power and demonstrate that our language is able to express common and complex queries needed for logical reasoning and data analytic tasks, especially related to KG reasoning.
\begin{example}
Let $\lprog = \set{A(\mathbf{t}, 3), A(\mathbf{s}, 5)}$. The immediate consequence operator of Datalog would derive: $I_\lprog = \lprog$. 
On the other hand, let  $\lprog = \set{A(\mathbf{t}, 3), A(\mathbf{s}, max(5))}$, then every model must contain $\set{A(\mathbf{t}, 3),A(\mathbf{s}, max(k))}, \forall k \leq 5$. Thus, we can finitely represent the interpretation with bounded facts. 
\end{example}

The limitations of Datalog (and \DatalogPm{} languages) are arithmetic computations (e.g., $1+2$ or counting) by definition. A classical example in data analytics is computing the shortest paths in a graph as demonstrated above. 
In a similar fashion, the language can express several other common data analytic queries that involve arithmetic computations such as counting paths between pairs of vertices in a DAG, computing relations in social networks, etc. In fact, it can express all queries limit \DatalogZ{} can. Additionally to such arithmetic tasks, our extension, Warded Bound \DatalogZ{}, contains Warded \DatalogPm{} and is hence capable of advanced reasoning tasks which require existentials in rule heads (for which limit \DatalogZ{} is by definition too weak).

Additionally to the examples above, let us highlight that our language contains Warded \DatalogPm{}. Thus, it can express all Knowledge Graph reasoning use cases Warded \DatalogPm{} can express by definition, and additionally has the capabilities of efficient arithmetic. 

\begin{example}
The following rules are Warded Bound \DatalogZ{} rules by the definition in Section \ref{ssec:synt-fragment}.
\begin{align*}
    R(s, x), P(\mathbf{t}, 3) \rightarrow & \exists \nu R(\nu, 3) \\
    S(s, y), P(\mathbf{t}, z) \rightarrow & R(s, max(y+z))
\end{align*}
The following rule is not, as the existental variable $z$ may not appear in a numeric position of predicate $P$:
\begin{align*}
    P(\mathbf{t}, 3) \rightarrow \exists z P(\mathbf{t}, z)
\end{align*}
For an interpretation $I = \set{R(a, 1), P(b,3), S(a, 2)}$ the first two rules derive (in the first iteration): $\set{R(n,3), R(a, max(3))}$ etc.
\end{example}
\begin{example}
We extend the shortest paths example with the notion of formalizing and computing the intuitive meaning of graph reachability. Warded Bound \DatalogZ{} can express the example as follows:
\begin{align}
    \mathit{Path}(start, min(0)).  \\
    \mathit{Path}(v, min(x)), \mathit{Edge}(v,w,min(y)) \nonumber \\
    \rightarrow path(w, min(x+y)).\\
    \mathit{Path}(v, min(x)), \mathit{Edge}(v,w,min(y)) \nonumber \\
    \rightarrow \exists u \; \mathit{Reachable}(start, u, min(x+y)).
\end{align}
Note that for the bound operator $min(k)$, $\mathit{Reachable}(start, u, min(k))$ implies that $\mathit{Reachable}(start, u, min(k')), k' \geq k$ holds thereby formalizing the shortest reachable path notion.
\end{example}
\bibliographystyleappendix{named}
\bibliographyappendix{main}
\end{document}